\documentclass[letterpaper, 10pt, conference]{ieeeconf} 
\IEEEoverridecommandlockouts
\usepackage{graphicx}
\usepackage{xcolor}
\usepackage{amsmath}
\usepackage{mathtools}
 \mathtoolsset{showonlyrefs}
\usepackage{tikz}
\usepackage{verbatim}
\usepackage{subcaption}
\usepackage{pgfplots}
\pgfplotsset{compat=1.10}
\usepackage{epstopdf} 
\usepackage{relsize}
\usepackage{booktabs}
\definecolor{Gray}{gray}{0.90}
\usepackage{colortbl}

\usepackage{algorithm}
\usepackage{algpseudocode}
\usepackage{color}
\usepackage{blkarray}
\usepackage{sidecap}
\usepackage{multirow}

\newtheorem{theorem}{\bf{Theorem}}[section]

\newtheorem{prop}[theorem]{\bf{Proposition}}

\newtheorem{lemma}[theorem]{\bf{Lemma}}

\newenvironment{definition}[1][Definition]{\begin{trivlist}
\item[\hskip \labelsep {\bfseries #1}]}{\end{trivlist}}

\usepackage{amssymb}
\usepackage{url}

\definecolor{amber}{rgb}{1.0, 0.75, 0.0}
\newcommand\figref{Fig.~\ref}

\usepackage{hyperref}
\usepackage{resizegather}

\usepackage{makecell}
\usepackage{censor}

\allowdisplaybreaks
\usepackage{stfloats}

\usepackage{bm}
\usepackage{amsmath}
\let\labelindent\relax
\usepackage{enumitem}

\DeclareMathOperator{\atan2}{atan2}

\providecommand{\secref}[1]{Sec.~\ref{#1}}
\providecommand{\appref}[1]{Appendix~\ref{#1}}
\providecommand{\figref}[1]{Fig.~\ref{#1}}
\providecommand{\tabref}[1]{Table~\ref{#1}}

\providecommand{\propref}[1]{Proposition~\ref{#1}}
\providecommand{\thmref}[1]{Theorem~\ref{#1}}
\providecommand{\lemref}[1]{Lemma~\ref{#1}}

\begin{document}{

\title{ \bf SC3BF: Shifted Collision Cone Control Barrier Function for Dynamic Obstacle Avoidance }
\author{ Amin Kashiri and Yasin Yaz{\i}c{\i}o\u{g}lu
\thanks{This work has been submitted to the IEEE for possible publication.
Copyright may be transferred without notice, after which this version may
no longer be accessible.}
\thanks{A. Kashiri is with the Department of Electrical and Computer Engineering at Northeastern University, Boston, MA ({\tt kashiri.a@northeastern.edu}).}%
\thanks{Y. Yaz{\i}c{\i}o\u{g}lu is with the Departments of Mechanical and Industrial Engineering and Electrical and Computer Engineering at Northeastern University, Boston, MA  ({\tt y.yazicioglu@northeastern.edu}).}%
}
}

\maketitle
\overrideIEEEmargins   
\begin{abstract}
The collision cone used by velocity-space control barrier functions is conservative: it rejects every relative velocity aimed into an obstacle, however slow. We propose the \emph{shifted collision-cone CBF} (SC3BF), which adds a state-dependent \emph{allowance} to the cone condition, so the robot may approach the obstacle at a rate that grows with distance and with its own speed. SC3BF is enforced by an ordinary quadratic program, and its safe set is forward invariant under bounded inputs without a minimum forward speed or a clearance margin. We prove that a nonzero allowance preserving safety always exists, and derive one in closed form. Against three velocity-space baselines on a kinematic bicycle among up to $100$ moving obstacles, SC3BF reaches the goal more often and modifies the nominal input less than half as much.
\end{abstract}

\section{Introduction}
\label{intro}
 
A vehicle navigating among moving obstacles must avoid collisions while progressing toward its goal, subject to kinematic constraints and input limits. A common approach combines a nominal controller with a safety filter that modifies its input as little as possible. Control Barrier Functions (CBFs) \cite{ames2017control}  represent the safe set as the region where a scalar function of the state is nonnegative, and constrain how fast that function can decrease. For dynamics that are affine in the input, this is a linear inequality on the input, which can be enforced with a small quadratic program (QP).
 
The QP is cheap to solve, but finding a suitable barrier function is difficult. For an acceleration-controlled vehicle, the distance derivative does not contain acceleration, so a distance-based barrier cannot impose braking. Braking-distance barriers \cite{wang2017safety} bring the acceleration into the inequality, but certify only the brake maneuver, and therefore reject states from which the vehicle could in fact pass safely using steering. High-order CBFs \cite{xiao2022high} instead differentiate the distance a second time, so that the acceleration appears, but certify a subset of the original safe set and introduce auxiliary class-$\mathcal{K}$ functions that are difficult to tune.
 
A second line of work builds the barrier in the space of relative velocities, where the input appears after one differentiation of the barrier function. Following the classical velocity obstacle \cite{fiorini1998motion}, the Collision-Cone CBF (C3BF) \cite{thontepu2023collision} keeps the relative velocity outside the cone of velocities aimed into the obstacle. The result is a valid barrier for the unicycle and the bicycle dynamics, and \cite{haraldsen2024safety} applies the velocity-obstacle principle with CBFs to unicycle-type vehicles that must keep a positive forward speed. The cone is conservative, however, since it rejects every velocity aimed into the obstacle, however slow. The Dynamic Parabolic CBF (DPCBF) \cite{park2026beyond} reduces that conservatism by replacing the cone with a parabolic region whose boundary depends on the relative velocity, so a slow approach is allowed where a fast one is not. However, as \secref{sec:background} sets out, the parabolic region rests on minimum-speed and clearance assumptions that their controller does not enforce, and neither construction covers the singularity of the barrier at contact.
 
The safe set need not be described by a single barrier function. Boolean compositions of several barriers are treated in \cite{glotfelter2017nonsmooth}, combinatorial compositions in \cite{ong2025combinatorial}, and the feasibility of compositions under input constraints in \cite{breeden2023compositions}. The velocity-obstacle-based CBF (VOCBF) \cite{huang2025dynamic} has this structure. It writes the outside of the cone as the union of the two half-spaces its edges define, and enforces a disjunction of the two edge-wise barrier inequalities, with binary variables selecting which holds. Its mixed-integer formulation scales poorly in the number of obstacles.
 
We introduce a velocity-space barrier for the kinematic bicycle, the \emph{shifted collision-cone CBF} (SC3BF), which adds a state-dependent \emph{allowance} to the cone condition, so the robot may approach the obstacle at a rate that grows with distance and with its own speed. The resulting filter needs neither a minimum forward speed nor a clearance margin. Our contributions are as follows:
\begin{enumerate}
\item We construct the barrier and the quadratic program that enforces it, and prove that the safe set is forward invariant under bounded inputs, including at contact.
\item We prove, for the kinematic bicycle, that for every choice of the bicycle's and the obstacle's parameters some nonzero allowance preserves safety, and derive one in closed form. We also introduce a certification procedure that yields substantially larger allowances for a given system.
\item We compare SC3BF to C3BF, DPCBF and VOCBF in simulation, among up to $100$ moving obstacles. SC3BF reaches the goal more often at every obstacle count and modifies the nominal input less than half as much as any of them.
\end{enumerate}

\section{Background}
\label{sec:background}
\subsection{Control Barrier Functions and the CBF-QP}
 
Consider
\begin{equation}
    \dot{x} = f(x) + g(x) u, \quad x \in \mathcal{X} \subseteq \mathbb{R}^n, \quad u \in \mathcal{U} \subseteq \mathbb{R}^m,
\end{equation}
where $f$ and $g$ are locally Lipschitz and $\mathcal{U}$ is the (possibly bounded) admissible input set. A safe set is the zero-superlevel set $\mathcal{C} = \{ x \in \mathcal{X} : h(x) \geq 0 \}$ of a continuously differentiable function $h : \mathcal{X} \to \mathbb{R}$.
 
The function $h$ is a Control Barrier Function (CBF) on $\mathcal{C}$ if there exists an extended class-$\mathcal{K}_\infty$ function $\alpha$ such that
\begin{equation}
    \sup_{u \in \mathcal{U}} \left[ L_f h(x) + L_g h(x) u + \alpha(h(x)) \right] \geq 0 \quad \forall x \in \mathcal{X},
    \label{eq:cbf_condition}
\end{equation}
where $L_f h$ and $L_g h$ are the Lie derivatives of $h$ along $f$ and $g$. Under this \emph{CBF condition}, any locally Lipschitz controller $u(x)$ satisfying the \emph{barrier condition}
\begin{equation}
    L_f h(x) + L_g h(x) u(x) \geq -\alpha(h(x))
    \label{eq:cbf_inequality}
\end{equation}
keeps the system inside $\mathcal{C}$ \cite{ames2017control}.
 
Since \eqref{eq:cbf_inequality} is affine in $u$, the CBF can be used as a safety filter on a nominal controller $u_\text{nom}(x)$. The filtered control is obtained by solving the quadratic program
\begin{equation}
    \begin{aligned}
        u^* = \arg\min_{u \in \mathcal{U}} \quad & \| u - u_\text{nom}(x) \|^2 \\
        \text{s.t.} \quad & L_f h(x) + L_g h(x) u \geq -\alpha(h(x)).
    \end{aligned}
    \label{eq:cbf_qp}
\end{equation}
When the safe set is an intersection of several such sets, an inequality of the form \eqref{eq:cbf_inequality} is added for each.
 
\subsection{Bicycle Dynamics, Obstacle Model, and Line-of-Sight Frame}
\label{subsec:bicycle}

We consider the kinematic bicycle model \cite{polack2017kinematic}. Its state is $\bm{x} = (x, y, \theta, v)^\top$, where $\bm{p} = (x, y)^\top$ is the position of the center of mass, $\theta$ the heading and $v$ the forward speed, and its inputs are the acceleration $a$ and the front-wheel steering angle $\delta$. With $\ell_f$ and $\ell_r$ the distances from the center of mass to the front and rear axles, the steering angle determines the slip angle $\beta = \tan^{-1}\bigl(\ell_r\tan\delta/(\ell_f + \ell_r)\bigr)$, the angle between the velocity of the center of mass and the heading, and the model is $\dot x = v\cos(\theta + \beta)$, $\dot y = v\sin(\theta + \beta)$, $\dot\theta = (v/\ell_r)\sin\beta$ and $\dot v = a$. As in \cite{thontepu2023collision,park2026beyond}, we assume that $\beta$ is small, so that $\cos\beta \approx 1$ and $\sin\beta \approx \beta$, and take $\bm{u} = (a, \beta)^\top$ as the input. This gives the control-affine model
\begin{equation}
\underbrace{\begin{bmatrix} \dot x \\ \dot y \\ \dot\theta \\ \dot v \end{bmatrix}}_{\dot{\bm{x}}}
= \underbrace{\begin{bmatrix} v\cos\theta \\ v\sin\theta \\ 0 \\ 0 \end{bmatrix}}_{f(\bm{x})}
+ \underbrace{\begin{bmatrix} 0 & -v\sin\theta \\ 0 & v\cos\theta \\ 0 & v/\ell_r \\ 1 & 0 \end{bmatrix}}_{g(\bm{x})}
\underbrace{\begin{bmatrix} a \\ \beta \end{bmatrix}}_{\bm{u}} .
\label{eq:fg}
\end{equation}
The input $\beta$ turns the heading and also moves the center of mass sideways: with $\bm{v} = (v\cos\theta, v\sin\theta)^\top$ the velocity along the heading, the first two rows of \eqref{eq:fg} give $\dot{\bm{p}} = \bm{v} + v\beta(-\sin\theta, \cos\theta)^\top$. The admissible input set is $\mathcal{U} = \{(a, \beta)^\top : |a| \le a_{\max},\ |\beta| \le \beta_{\max}\}$.
 
A single dynamic obstacle is modeled as a disk of radius $r_o$ at position $\bm{p}_o(t) \in \mathbb{R}^2$ moving with constant inertial velocity $\bm{v}_o$ with magnitude bounded by $v_{o,\max}$. The combined safety radius is $r = r_r + r_o$, where $r_r$ is the robot's radius. Relative position and velocity are $\Delta\bm{p} = \bm{p}_o - \bm{p}$ and $\Delta\bm{v} = \bm{v} - \bm{v}_o$, with $\bm{v}$ rather than $\dot{\bm{p}}$, so that the barrier depends on the state alone. Since the barrier depends on the obstacle, we treat $\bm{p}_o$ as part of the state, writing $x := (\bm{x}, \bm{p}_o)$ for the augmented state, with $\dot{\bm{p}}_o = \bm{v}_o$, and take Lie derivatives along these augmented dynamics, as in \cite{thontepu2023collision}.
We write $\rho = \|\Delta\bm{p}\|$ for the relative distance, and $d = \sqrt{\rho^2 - r^2}$ for the tangent distance. The forward speed is taken to lie in $[0, v_{\max}]$, an assumption used only by the feasibility result of \secref{subsec:results}. To bound the worst-case approach speed we define $w_{\max} := v_{\max} + v_{o,\max}$, so that $\|\Delta\bm{v}\| \le w_{\max}$.
 
The \emph{line-of-sight (LoS) frame} is the frame rotated by $\gamma = \atan2(\Delta p_y, \Delta p_x)$ so that its first axis is aligned with $\Delta\bm{p}$. With $\mathbf{R}(\cdot)$ the planar rotation matrix, the relative velocity and the heading in this frame are
\begin{equation}
\Delta\tilde{\bm{v}} = \mathbf{R}(-\gamma)\,\Delta\bm{v} =: \begin{bmatrix} \Delta v_\parallel \\ \Delta v_\perp \end{bmatrix},
\qquad
\tilde\theta = \theta - \gamma .
\label{eq:los}
\end{equation}
Thus $\Delta v_\parallel$ is the closing rate, the component of $\Delta\bm{v}$ along the line of sight, which is positive whenever $\Delta\bm{v}$ makes an acute angle with $\Delta\bm{p}$. The transverse component is $\Delta v_\perp$, and $\tilde\theta$ is the heading measured from the line of sight.

\subsection{Related Work: Velocity-Space CBFs}
 
A direct distance-based CBF $h = \rho^2 - r^2$ cannot be used in the CBF-QP \eqref{eq:cbf_qp} without modification. Its rate $2\rho\dot\rho$ does not contain the acceleration (see \appref{app:rate}). To avoid this, recent work \cite{thontepu2023collision,haraldsen2024safety,park2026beyond,huang2025dynamic} maps the obstacle constraint from configuration space into the space of relative velocities, where the constraint acquires relative degree one.

The Collision-Cone CBF (C3BF) \cite{thontepu2023collision} uses the observation that, under constant-velocity motion, collision occurs only if the relative velocity vector points into the obstacle safety disk. The set of such directions forms a cone in velocity space of half-angle $\arcsin(r / \rho)$, called the collision cone. The robot is safe as long as $\Delta \bm{v}$ stays outside this cone, and each obstacle contributes a cone of its own. This condition can be written as
\begin{equation}
    h_\text{C3BF}(\Delta \bm{p}, \Delta \bm{v}) = -\Delta \bm{p}^\top \Delta \bm{v} + \|\Delta \bm{v}\| \sqrt{\|\Delta \bm{p}\|^2 - r^2} \geq 0,
    \label{eq:c3bf}
\end{equation}
and \cite{thontepu2023collision} shows that $h_\text{C3BF}$ has relative degree one for the kinematic bicycle, so the CBF-QP can be applied directly.
 
C3BF is conservative because its cone half-angle depends only on $\|\Delta \bm{p}\|$ and $r$. The condition therefore ignores the magnitude of $\Delta \bm{v}$. DPCBF \cite{park2026beyond} reduces this by replacing each cone with a parabolic region in the same $v$-space, whose vertex and curvature depend on both $\|\Delta \bm{p}\|$ and $\|\Delta \bm{v}\|$. Its result is established under two assumptions:
\begin{enumerate}[label=\textbf{(A\arabic*)}, leftmargin=*]
\item \emph{Forward speed.} $v \in [v_{\min}, v_{\max}]$ with $v_{\max} > v_{\min} > 0$.
\label{ass:vmin}
\item \emph{Clearance margin.} There exist $s > 1$ and a finite sensing range $p_{\max}$ such that $sr \le \rho \le p_{\max}$, so $d \geq d_{\min} = r\sqrt{s^2 - 1} > 0$.
\label{ass:clearance}
\end{enumerate}
\ref{ass:vmin} excludes the degeneracy $v = 0$ at which the steering authority $v/\ell_r$ vanishes. \ref{ass:clearance} is assumed rather than enforced by the controller, which yields $h \ge 0$ and not $\rho \ge sr$. Under \ref{ass:vmin}--\ref{ass:clearance}, \cite{park2026beyond} shows that at every state on the boundary of the safe set some input in $\mathcal{U}$ satisfies the barrier condition \eqref{eq:cbf_inequality}, one constraint per obstacle.
 
A third line of work replaces the single barrier by a disjunction. VOCBF \cite{huang2025dynamic} imposes the barrier condition on at least one of the two half-spaces whose union is the outside of the cone, so its safe set is that of C3BF. Besides the mixed-integer program, \cite{huang2025dynamic} proposes a decomposition into quadratic programs and a learned selection. We compare against the mixed-integer program.
 
These methods leave three gaps. C3BF is conservative, since its condition depends only on the direction of $\Delta\bm{v}$. DPCBF removes much of that conservatism but rests on \ref{ass:vmin}--\ref{ass:clearance}, neither of which is enforced by the controller proposed in \cite{park2026beyond}. VOCBF needs neither assumption, but the active half-space is chosen by the solver rather than by the barrier, and its formulation is a mixed-integer program. Finally, C3BF, DPCBF and VOCBF all contain $d = \sqrt{\rho^2 - r^2}$, whose derivative $\dot d = \rho\dot\rho/d$ is singular at contact: \cite{thontepu2023collision} verifies the CBF condition \eqref{eq:cbf_condition} only where $d > 0$, \cite{huang2025dynamic} assumes that the robot is not in contact with the obstacle, and \cite{park2026beyond} bounds the term by \ref{ass:clearance}.

\section{Proposed Method}
\label{sec:method}
 
\subsection{Proposed Barrier Function}
\label{subsec:def}
 
A distant obstacle approached slowly can often still be avoided, yet C3BF declares any state with $\Delta\bm{v}$ inside the collision cone unsafe regardless of the obstacle distance or the approach speed \cite{park2026beyond}. We therefore allow the robot to approach the obstacle at any rate up to $c$, imposing the cone condition only on the excess $\Delta v_\parallel - c$. Geometrically this slides the apex of the cone away from the origin of the relative-velocity plane, so wherever $c > 0$ motion toward the obstacle is no longer immediately unsafe. The \emph{allowance}
\begin{equation}
c := \kappa\, d^{\,p} v^2, \qquad \kappa \ge 0, \quad p \ge 1,
\label{eq:allowance}
\end{equation}
vanishes at contact and at rest, and grows with distance and robot speed.
 
The cone has two edges, indexed by $\sigma \in \{+1, -1\}$, as shown in \figref{fig:cone}. We define one function for each and take their maximum as the barrier function. With $\kappa = 0$ the two functions describe the edges of the C3BF cone, and the barrier has the same safe set as C3BF. Since $c \ge 0$, a positive $\kappa$ can only enlarge this set, which gives a less conservative safe set, as illustrated in \figref{fig:cone}.
 
\begin{figure}[!t]
    \centering
    \includegraphics[width=0.72\linewidth]{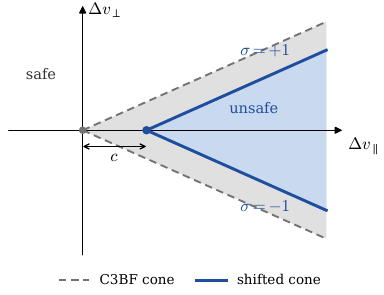}
    \caption{The C3BF cone (dashed) and the SC3BF cone (solid) in the relative-velocity plane, in the line-of-sight frame. The allowance $c$ moves the apex along the closing direction, and the gray region between the two cones is the conservatism it removes. }
    \label{fig:cone}
\end{figure}
 
\begin{definition}[SC3BF]
\label{def:barrier}
For $\sigma \in \{+1,-1\}$ let
\begin{equation}
\begin{aligned}
h_\sigma &:= -r\bigl(\Delta v_\parallel - c\bigr) + \sigma\, d\, \Delta v_\perp \\
&\phantom{:}= -r \Delta v_\parallel + \sigma d \Delta v_\perp + r \kappa d^{\,p} v^2 ,
\end{aligned}
\label{eq:hsigma}
\end{equation}
and set $h := \max(h_{+1}, h_{-1})$ and $\mathcal{C} := \{x : \rho \ge r,\ h(x) \ge 0\}$.
\end{definition}
 
\subsection{Proposed Controller}
\label{subsec:controller}
 
For $\rho > r$, differentiating $h_\sigma$ along the augmented dynamics gives, as derived in \appref{app:rate},
\begin{equation}
\begin{aligned}
\dot h_\sigma &= L_f h_\sigma + a\bigl(c^a_\sigma + S\bigr) + \beta B_\sigma ,\\
L_f h_\sigma &= \frac{r \Delta v_\perp^2}{\rho} - \frac{r \Delta v_\parallel}{d}\Bigl(\frac{\sigma r \Delta v_\perp}{\rho} + \kappa p\, d^{\,p-1} \rho v^2\Bigr), \\
c^a_\sigma &= -r\cos\tilde\theta + \sigma d \sin\tilde\theta, \qquad
c^\beta_\sigma = r \sin\tilde\theta + \sigma d \cos\tilde\theta ,\\
B_\sigma &= \tfrac{v^2}{\ell_r}c^\beta_\sigma + \tfrac{v\rho}{d}\bigl(\sigma\Delta v_\perp + r\kappa p\, d^{\,p-1}v^2\bigr)\sin\tilde\theta \\
&\quad + \tfrac{v}{\rho}\bigl(r\Delta v_\perp + \sigma d\,\Delta v_\parallel\bigr)\cos\tilde\theta ,
\end{aligned}
\label{eq:hdot}
\end{equation}
with $S = 2 r \kappa d^{\,p} v$. The last two terms of $B_\sigma$ come from the sideways motion of the center of mass. 
The inequality $\dot h_\sigma \ge -\alpha(h_\sigma)$ cannot be imposed at contact, where $\rho = r$, because the formula \eqref{eq:hdot} is singular there: $L_f h_\sigma$ contains $r\Delta v_\parallel/d$. Instead, the controller requires a nondecreasing distance and a nonincreasing closing rate, two constraints affine in $\bm{u}$ derived in \appref{app:contact}. Imposing $\dot h_\sigma \ge -\alpha(h_\sigma)$ for $\rho > r$ and these \emph{contact constraints} for $0 < \rho \le r$ gives
\begin{equation}
\begin{aligned}
\bm{u}^* = \arg\min_{\bm{u} \in \mathcal{U}} \;\; & \|\bm{u} - \bm{u}_\text{nom}\|^2 \\
\text{s.t.} \;\; &
\begin{cases}
\dot h_\sigma(x, \bm{u}) \ge -\alpha(h_\sigma) & \text{if } \rho > r, \\[3pt]
\dot\rho \ge 0, \ \ \tfrac{d}{dt}\Delta v_\parallel \le 0 & \text{if } 0 < \rho \le r,
\end{cases}
\end{aligned}
\label{eq:ours_qp}
\end{equation}
where $\alpha$ is an extended class-$\mathcal{K}$ function, $\sigma = +1$ if $\Delta v_\perp \ge 0$ and $\sigma = -1$ otherwise, which selects the edge attaining the maximum since $h_{+1} - h_{-1} = 2d\,\Delta v_\perp$. Because the sign of $\Delta v_\perp$ fixes the active edge, \eqref{eq:ours_qp} is an ordinary quadratic program.
 
\subsection{Main Results}
\label{subsec:results}
 
We call a pair $(\kappa, p)$ \emph{admissible} if, for every extended class-$\mathcal{K}$ function $\alpha$, \eqref{eq:ours_qp} is feasible at every state of $\mathcal{C}$. We now show that admissible pairs exist, and that the solutions of \eqref{eq:ours_qp} prevent collisions.
 
At each state with $\rho > r$ the program enforces $\dot h_\sigma \ge -\alpha(h_\sigma)$ for only the edge selected by the sign of $\Delta v_\perp$. Write $\mathcal{C}^\sigma := \{h_\sigma \ge 0\} \cap \{\sigma\Delta v_\perp \ge 0\}$ for the part of $\mathcal{C}$ on which edge $\sigma$ attains the maximum. Together, $\mathcal{C}^{+1}$ and $\mathcal{C}^{-1}$ cover $\mathcal{C}$, so feasibility needs to be established only for edge $\sigma$ on $\mathcal{C}^\sigma$. In the following proposition, we show that there always exists a pair $(\kappa, p)$ for which \eqref{eq:ours_qp} is feasible at every state of $\mathcal{C}$ with $\rho > r$.
 
\begin{prop}[Existence of an admissible pair]
\label{prop:A}
For every positive safety radius $r$, speed bound $v_{\max}$, input bounds $a_{\max}$ and $\beta_{\max}$ and rear-axle distance $\ell_r$, and every obstacle speed bound $v_{o,\max} \ge 0$, there exists a pair $(\kappa, p)$ with $\kappa > 0$ such that, for every extended class-$\mathcal{K}$ function $\alpha$, \eqref{eq:ours_qp} is feasible at every state of $\mathcal{C}$ with $\rho > r$.
\end{prop}
 
\emph{Proof:} See \appref{app:gain}.
 
At contact, where $\rho = r$, the program imposes the contact constraints instead, and $\bm{u} = 0$ satisfies them (\appref{app:contact}). Hence \eqref{eq:ours_qp} is feasible at every state of $\mathcal{C}$.
 
Feasibility at each state does not by itself show that $\mathcal{C}$ is forward invariant. The standard result of \cite{ames2017control} assumes a single continuously differentiable $h$, which $h = \max(h_{+1}, h_{-1})$ is not, so \thmref{thm:safety} establishes invariance of $\mathcal{C}$ directly.
 
\begin{theorem}[Safety]
\label{thm:safety}
Let $x(0) \in \mathcal{C}$. If the applied input satisfies the constraints of \eqref{eq:ours_qp} at every time, then $x(t) \in \mathcal{C}$, and hence $\rho(t) \ge r$, for all $t \ge 0$.
\end{theorem}
 
\emph{Proof:} See \appref{app:safety}.
 
With \propref{prop:A}, \thmref{thm:safety} keeps \eqref{eq:ours_qp} feasible along the trajectory: inputs that satisfy its constraints keep the state in $\mathcal{C}$, where the program is feasible.
 
\subsection{Choosing the Allowance}
\label{subsec:choosing}
 
\propref{prop:A} shows that a pair $(\kappa,p)$ exists such that our controller \eqref{eq:ours_qp} is always feasible, but does not say which value to use. At $p = 1$, \appref{app:gain} supplies a threshold $\kappa^\star$ in closed form, and every pair $(\kappa, 1)$ with $0 < \kappa \le \kappa^\star$ is admissible, so a usable gain requires no numerical search. This threshold is conservative because its derivation bounds $L_f h_\sigma$ and the contribution of the input separately, each at its own worst case. For a given system, the numerical procedure of \appref{app:certify} can verify that \eqref{eq:ours_qp} remains feasible with a larger $\kappa$ at every state of a specified domain, rather than only at sampled states.
 
The procedure certifies one pair $(\kappa, p)$ at a time, and every certified pair keeps \eqref{eq:ours_qp} feasible on the certified domain, so the guarantees of \secref{subsec:results} hold for each. Within this range, $\kappa$ is a design choice: a larger gain makes the safe set of each obstacle less conservative, but also lets the robot approach obstacles faster. The exponent $p$ sets how quickly the allowance grows with distance. We evaluate $p = 1$ and $p = 1.2$.

\section{Evaluation}
\label{sec:eval}
 
\subsection{Setup and Metrics}
 
The robot is the kinematic bicycle of \secref{subsec:bicycle} with $a_{\max} = 5$~m/s$^2$, $\beta_{\max} = 0.28$~rad, $\ell_r = 0.2$~m, $v_{\max} = 3.5$~m/s and radius $0.3$~m, integrated at $50$~ms. It drives from $(1, 7.5)$ to a goal at $(25, 7.5)$~m, through $20$, $50$ or $100$ obstacles of radius $0.7$~m placed uniformly at random in $[5, 30] \times [-5, 20]$~m, each moving in a random direction at a constant speed drawn from $[0, 1.2]$~m/s. Every method enforces its barrier at the inflated radius $R_\text{eff} = 1.05\,r$, which leaves a margin the continuous-time analysis does not model but a discrete time step needs. A run is cut off after $2000$ steps.  The robot starts at $1$~m/s, heading toward the goal. Every method uses the same nominal controller, the goal-tracking controller of the framework of \cite{park2026beyond}, solves its program with Gurobi, and is given the same $100$ random instances.
 
DPCBF uses the gains reported for this obstacle radius, $k_\lambda = 0.144$ and $k_\mu = 0.505$. For SC3BF we use $\kappa = 0.18$ at $p = 1$ and $\kappa = 0.1672$ at $p = 1.2$, which the procedure of \appref{app:certify} verifies at the enforced radius $R_\text{eff}$, the first for every $\rho > R_\text{eff}$ and the second for $\rho \le 70$~m. The extended class-$\mathcal{K}$ function of \eqref{eq:cbf_condition} is taken linear, $\alpha(h) = k_\alpha h$, and each method is given the $k_\alpha$ with the highest success rate at each obstacle count, from a grid between $1.5$ and $40$.
 
A run succeeds if the robot comes within $0.3$~m of the goal. It fails, ending the run, if the quadratic program is infeasible or cannot be solved, if the robot overlaps any obstacle at the end of a step, which we record as a collision, or if the step limit is reached first. We also report the \emph{intervention}, which measures how far the filter moves the nominal input. At each step it is
\begin{equation}
\sqrt{\Bigl(\tfrac{a^* - a_\text{nom}}{a_{\max}}\Bigr)^{2} + \Bigl(\tfrac{\beta^* - \beta_\text{nom}}{\beta_{\max}}\Bigr)^{2}} ,
\label{eq:interv}
\end{equation}
so that each channel is measured relative to its input bound. For every method, we average the intervention over the steps of each successful run, and then over the successful runs.

\subsection{Results}
\label{subsec:results_exp}
 
\begin{table}[!t]
\centering
\caption{Success rate and intervention over 100 random instances per obstacle count. Intervention is averaged over successful runs. SC3BF is our method.}
\label{tab:results}
\footnotesize
\setlength{\tabcolsep}{3.8pt}
\renewcommand{\arraystretch}{1.05}
\begin{tabular}{lcccccc}
\toprule
& \multicolumn{2}{c}{20 obstacles} & \multicolumn{2}{c}{50 obstacles}
& \multicolumn{2}{c}{100 obstacles} \\
\cmidrule(lr){2-3}\cmidrule(lr){4-5}\cmidrule(lr){6-7}
Method & succ. & interv. & succ. & interv. & succ. & interv. \\
\midrule
C3BF \cite{thontepu2023collision}  & 22\,\% & 0.470 & 6\,\%  & 0.440 & 0\,\%  & --    \\
VOCBF \cite{huang2025dynamic}      & 53\,\% & 0.119 & 38\,\% & 0.488 & 12\,\% & 0.724 \\
DPCBF \cite{park2026beyond}        & 84\,\% & 0.357 & 48\,\% & 0.542 & 17\,\% & 0.684 \\
\midrule
SC3BF, $p = 1$                     & 91\,\% & \textbf{0.052} & 62\,\% & \textbf{0.120} & 22\,\% & \textbf{0.222} \\
SC3BF, $p = 1.2$                   & \textbf{92\,\%} & 0.053
                                   & \textbf{67\,\%} & 0.131
                                   & \textbf{27\,\%} & 0.229 \\
\bottomrule
\end{tabular}
\end{table}
 
With $p = 1.2$, SC3BF reaches the goal more often than every baseline at all three obstacle counts (\tabref{tab:results}). Its success rate exceeds DPCBF's by $8$, $19$ and $10$ percentage points: $92\%$ versus $84\%$, $67\%$ versus $48\%$, and $27\%$ versus $17\%$. Among successful runs, its intervention is less than half of each baseline's at every count. 

\section{Conclusion}
\label{sec:conclusion}
 
We introduced SC3BF, a velocity-space control barrier function that relaxes the
collision-cone condition by an allowance growing with distance and robot speed. The
construction is enforced by an ordinary quadratic program and keeps its safe set
forward invariant under bounded inputs, including at contact. SC3BF needs neither a
clearance margin nor a minimum forward speed, which are assumed in DPCBF's theoretical guarantees. We showed that, for any choice of the bicycle and
obstacle parameters, there exists a nonzero allowance that preserves safety. In simulations with up to $100$ moving obstacles, SC3BF reached the goal more often than C3BF, DPCBF and VOCBF at every obstacle count and modified the nominal input less than half as much as any of them.

\bibliographystyle{IEEEtran}
\bibliography{references}

@article{ames2017control,
  author={Ames, Aaron D. and Xu, Xiangru and Grizzle, Jessy W. and Tabuada, Paulo},
  title={Control Barrier Function Based Quadratic Programs for Safety Critical Systems},
  journal={IEEE Transactions on Automatic Control},
  volume={62},
  number={8},
  pages={3861--3876},
  year={2017}
}

@article{wang2017safety,
  author={Wang, Li and Ames, Aaron D. and Egerstedt, Magnus},
  title={Safety Barrier Certificates for Collisions-Free Multirobot Systems},
  journal={IEEE Transactions on Robotics},
  volume={33},
  number={3},
  pages={661--674},
  year={2017}
}

@article{xiao2022high,
  author={Xiao, Wei and Belta, Calin},
  title={High-Order Control Barrier Functions},
  journal={IEEE Transactions on Automatic Control},
  volume={67},
  number={7},
  pages={3655--3662},
  year={2022}
}

@article{glotfelter2017nonsmooth,
  author={Glotfelter, Paul and Cort\'es, Jorge and Egerstedt, Magnus},
  title={Nonsmooth Barrier Functions With Applications to Multi-Robot Systems},
  journal={IEEE Control Systems Letters},
  volume={1},
  number={2},
  pages={310--315},
  year={2017}
}

@article{fiorini1998motion,
  author={Fiorini, Paolo and Shiller, Zvi},
  title={Motion Planning in Dynamic Environments Using Velocity Obstacles},
  journal={The International Journal of Robotics Research},
  volume={17},
  number={7},
  pages={760--772},
  year={1998}
}

@inproceedings{thontepu2023collision,
  title={Collision cone control barrier functions for kinematic obstacle avoidance in ugvs},
  author={Thontepu, Phani and Goswami, Bhavya Giri and Tayal, Manan and Singh, Neelaksh and PI, Shyam Sundar and MG, Shyam Sundar and Sundaram, Suresh and Katewa, Vaibhav and Kolathaya, Shishir},
  booktitle={2023 Ninth Indian Control Conference (ICC)},
  pages={293--298},
  year={2023},
  organization={IEEE}
}

@inproceedings{haraldsen2024safety,
  author={Haraldsen, Aurora and Wiig, Martin S. and Ames, Aaron D. and Pettersen, Kristin Y.},
  title={Safety-Critical Control of Nonholonomic Vehicles in Dynamic Environments Using Velocity Obstacles},
  booktitle={2024 American Control Conference (ACC)},
  pages={3152--3159},
  year={2024}
}

@article{huang2025dynamic,
  author={Huang, Jihao and Zeng, Jun and Chi, Xuemin and Sreenath, Koushil and Liu, Zhitao and Su, Hongye},
  title={Dynamic Collision Avoidance Using Velocity Obstacle-Based Control Barrier Functions},
  journal={IEEE Transactions on Control Systems Technology},
  volume={33},
  number={5},
  pages={1601--1615},
  year={2025}
}

@inproceedings{park2026beyond,
  author={Park, Hun Kuk and Kim, Taekyung and Panagou, Dimitra},
  title={Beyond Collision Cones: Dynamic Obstacle Avoidance for Nonholonomic Robots via Dynamic Parabolic Control Barrier Functions},
  booktitle={2026 IEEE International Conference on Robotics and Automation (ICRA)},
  year={2026}
}

@inproceedings{breeden2023compositions,
  title={Compositions of multiple control barrier functions under input constraints},
  author={Breeden, Joseph and Panagou, Dimitra},
  booktitle={2023 American Control Conference (ACC)},
  pages={3688--3695},
  year={2023},
  organization={IEEE}
}

@article{ong2025combinatorial,
  title={Combinatorial control barrier functions: Nested boolean and p-choose-r compositions of safety constraints},
  author={Ong, Pio and Lee, Haejoon and Molnar, Tamas G and Panagou, Dimitra and Ames, Aaron D},
  journal={IEEE Control Systems Letters},
  volume={9},
  pages={2705--2710},
  year={2025},
  publisher={IEEE}
}

@inproceedings{polack2017kinematic,
  author    = {Polack, Philip and Altch{\'e}, Florent and d'Andr{\'e}a-Novel, Brigitte and de La Fortelle, Arnaud},
  title     = {The kinematic bicycle model: A consistent model for planning feasible trajectories for autonomous vehicles?},
  booktitle = {IEEE Intelligent Vehicles Symposium (IV)},
  year      = {2017},
  pages     = {812--818}
}

\appendices

\section{Rates of the Relative Motion and of an Edge}
\label{app:rate}
 
This appendix derives the rates used in the paper: $\dot\rho$ and $\dot d$, the rate $\dot h_\sigma$ of an edge in \eqref{eq:hdot}, and the rate of $\Delta v_\parallel$, which the contact constraints of \appref{app:contact} use. 
 
\emph{Velocity along the heading.} Since $\bm{v} {=} v\,(\cos\theta, \sin\theta)^\top$, the product rule gives $\dot{\bm{v}} {=} \dot v\,(\cos\theta, \sin\theta)^\top {+} v\dot\theta\,(-\sin\theta, \cos\theta)^\top$. We substitute $\dot v = a$ and $\dot\theta = v\beta/\ell_r$ from \eqref{eq:fg} to get
\begin{equation}
\dot{\bm{v}} = a\,(\cos\theta, \sin\theta)^\top + \tfrac{v^2}{\ell_r}\beta\,(-\sin\theta, \cos\theta)^\top ,
\label{eq:vdot}
\end{equation}
and since $\bm{v}_o$ is constant, $\tfrac{d}{dt}\Delta\bm{v} = \dot{\bm{v}}$.
 
\emph{Distance and line of sight.} We use two identities: for every angle $\phi$,
\[
\begin{aligned}
\mathbf{R}(-\gamma)(\cos\phi, \sin\phi)^\top &= (\cos(\phi - \gamma), \sin(\phi - \gamma))^\top ,\\
\mathbf{R}(-\gamma)(-\sin\phi, \cos\phi)^\top &= (-\sin(\phi - \gamma), \cos(\phi - \gamma))^\top .
\end{aligned}
\]
We write $\tfrac{d}{dt}\Delta\bm{p}$ in two ways and rotate both into the line-of-sight frame. First, the first two rows of \eqref{eq:fg} give $\dot{\bm{p}} = \bm{v} + v\beta(-\sin\theta, \cos\theta)^\top$, and $\dot{\bm{p}}_o = \bm{v}_o$, so $\tfrac{d}{dt}\Delta\bm{p} = -\Delta\bm{v} - v\beta(-\sin\theta, \cos\theta)^\top$. By \eqref{eq:los} and the second identity with $\phi = \theta$,
\[
\mathbf{R}(-\gamma)\tfrac{d}{dt}\Delta\bm{p} = \begin{bmatrix} -\Delta v_\parallel + v\beta\sin\tilde\theta \\ -\Delta v_\perp - v\beta\cos\tilde\theta \end{bmatrix} .
\]
Second, $\Delta\bm{p} = \rho\,(\cos\gamma, \sin\gamma)^\top$ gives $\tfrac{d}{dt}\Delta\bm{p} = \dot\rho\,(\cos\gamma, \sin\gamma)^\top + \rho\dot\gamma\,(-\sin\gamma, \cos\gamma)^\top$, and both identities with $\phi = \gamma$ give
\[
\mathbf{R}(-\gamma)\tfrac{d}{dt}\Delta\bm{p} = \begin{bmatrix} \dot\rho \\ \rho\dot\gamma \end{bmatrix} .
\]
Equating the two expressions row by row gives
\begin{equation}
\dot\rho = -\Delta v_\parallel + v\beta\sin\tilde\theta, \qquad \rho\dot\gamma = -\Delta v_\perp - v\beta\cos\tilde\theta .
\label{eq:rhodot}
\end{equation}
For $\rho {>} r$, taking the derivative of $d^2 {=} \rho^2 {-} r^2$ gives $\dot d {=} \rho\dot\rho/d$.
 
\emph{Rates of $\Delta v_\parallel$ and $\Delta v_\perp$.} By the product rule, $\tfrac{d}{dt}\bigl(\mathbf{R}(-\gamma)\Delta\bm{v}\bigr) = \mathbf{R}(-\gamma)\dot{\bm{v}} + \bigl(\tfrac{d}{dt}\mathbf{R}(-\gamma)\bigr)\Delta\bm{v}$. By the two identities with $\phi = \theta$, the first term is \eqref{eq:vdot} with $\theta$ replaced by $\tilde\theta$. For the second, differentiating $\mathbf{R}(-\gamma) = \bigl[\begin{smallmatrix} \cos\gamma & \sin\gamma \\ -\sin\gamma & \cos\gamma \end{smallmatrix}\bigr]$ gives $\tfrac{d}{dt}\mathbf{R}(-\gamma) = \dot\gamma\bigl[\begin{smallmatrix} -\sin\gamma & \cos\gamma \\ -\cos\gamma & -\sin\gamma \end{smallmatrix}\bigr]$. By \eqref{eq:los}, the second term is therefore $\dot\gamma\,(\Delta v_\perp, -\Delta v_\parallel)^\top$. We substitute $\dot\gamma$ from \eqref{eq:rhodot} to get
\begin{align}
\tfrac{d}{dt}\Delta v_\parallel &= a\cos\tilde\theta - \tfrac{v^2}{\ell_r}\beta\sin\tilde\theta - \tfrac{\Delta v_\perp(\Delta v_\perp + v\beta\cos\tilde\theta)}{\rho}, \label{eq:ddvpar}\\
\tfrac{d}{dt}\Delta v_\perp &= a\sin\tilde\theta + \tfrac{v^2}{\ell_r}\beta\cos\tilde\theta + \tfrac{\Delta v_\parallel(\Delta v_\perp + v\beta\cos\tilde\theta)}{\rho} . \nonumber
\end{align}
 
\emph{Rate of an edge.} For $\rho > r$, taking the derivative of \eqref{eq:hsigma}, with $\dot v = a$, gives
\[
\begin{aligned}
\dot h_\sigma = {}& -r\tfrac{d}{dt}\Delta v_\parallel + \sigma\bigl(\dot d\,\Delta v_\perp + d\,\tfrac{d}{dt}\Delta v_\perp\bigr) \\
& + r\kappa p\, d^{\,p-1}\dot d\, v^2 + 2r\kappa d^{\,p} v\,a .
\end{aligned}
\]
Substituting the three rates above, using $\rho^2 - d^2 = r^2$, and collecting the terms without input and the coefficients of $a$ and $\beta$ gives $L_f h_\sigma$, $c^a_\sigma + S$ and $B_\sigma$ in \eqref{eq:hdot}.
 
\section{The Contact Constraints}
\label{app:contact}
 
At contact, $\rho = r$, the rate \eqref{eq:hdot} cannot be used, since $L_f h_\sigma$ contains $r\Delta v_\parallel/d$ and $d = 0$. This appendix derives the constraints that \eqref{eq:ours_qp} imposes instead, and shows that they are feasible at every state of $\mathcal{C}$ at contact.

At contact $d = 0$, so both edges are equal to $-r\Delta v_\parallel$. Therefore, $h = -r\Delta v_\parallel$. Since $h \ge 0$ on $\mathcal{C}$, a state of $\mathcal{C}$ at contact has $\Delta v_\parallel \le 0$. While contact lasts, two things must hold. The robot must not enter the obstacle, so we require $\dot\rho \ge 0$. The state must also stay in $\mathcal{C}$, which at contact means that $\Delta v_\parallel$ stays nonpositive, so we require $\tfrac{d}{dt}\Delta v_\parallel \le 0$. Writing both requirements with \eqref{eq:rhodot} and \eqref{eq:ddvpar} gives
\[
\begin{aligned}
v\beta\sin\tilde\theta &\ \ge\ \Delta v_\parallel ,\\
a\cos\tilde\theta - \tfrac{v^2}{\ell_r}\beta\sin\tilde\theta - \tfrac{v\Delta v_\perp}{\rho}\beta\cos\tilde\theta &\ \le\ \tfrac{\Delta v_\perp^2}{\rho} ,
\end{aligned}
\]
which are the contact constraints of \eqref{eq:ours_qp}. Both are affine in $\bm{u}$. The input $\bm{u} = 0$ satisfies both, since it gives $\dot\rho = -\Delta v_\parallel \ge 0$ and $\tfrac{d}{dt}\Delta v_\parallel = -\Delta v_\perp^2/\rho \le 0$. We impose them for $0 < \rho \le r$, not only at $\rho = r$, because the proof of \thmref{thm:safety} argues by contradiction on an interval on which $\rho < r$.
 
\section{Proof of \propref{prop:A}}
\label{app:gain}
 
\propref{prop:A} asks for a pair $(\kappa, p)$ for which, for every extended class-$\mathcal{K}$ function $\alpha$, the constraint $\dot h_\sigma \ge -\alpha(h_\sigma)$ can be met at every state of $\mathcal{C}^\sigma$ with $\rho > r$. On $\mathcal{C}^\sigma$, $h_\sigma \ge 0$ and so $-\alpha(h_\sigma) \le 0$. It is therefore enough to show that some input makes $\dot h_\sigma$ nonnegative:
\begin{equation}
\sup_{\bm{u} \in \mathcal{U}} \dot h_\sigma(x, \bm{u}) \ \ge\ 0 \qquad \text{on } \mathcal{C}^\sigma \cap \{\rho > r\} .
\label{eq:key}
\end{equation}
Since $\dot h_\sigma$ in \eqref{eq:hdot} is affine in $\bm{u}$, its largest value is reached with each input at its bound, with the sign of its coefficient:
\begin{equation}
\sup_{\bm{u} \in \mathcal{U}} \dot h_\sigma = L_f h_\sigma + \Phi_\sigma,
\quad
\Phi_\sigma := a_{\max}\bigl|c^a_\sigma + S\bigr| + \beta_{\max}\bigl|B_\sigma\bigr| . \label{eq:sep}
\end{equation}
We call $L_f h_\sigma$ the \emph{drift}, since it does not depend on the input, and $\Phi_\sigma$ the \emph{authority}, the largest increase that acceleration and steering can add. The drift can be negative, and the proof shows that the authority always makes up for it.
 
If $\Delta v_\parallel \le 0$, both terms of $L_f h_\sigma$ in \eqref{eq:hdot} are nonnegative on $\mathcal{C}^\sigma$, so $\bm{u} = 0$ satisfies \eqref{eq:key}. The rest of this appendix takes $\Delta v_\parallel > 0$ and proceeds in three steps. \appref{app:bounds} bounds the drift and the authority from below, both in proportion to $\rho v^2$, which reduces \eqref{eq:key} to a condition on the parameters alone. \appref{app:closedform} gives a gain that satisfies this condition in closed form, and \appref{app:propproof} completes the proof.
 
\subsection{Bounds on the Drift and the Authority}
\label{app:bounds}
 
We state the bounds for a general exponent $p \ge 1$. The proof of \propref{prop:A} uses them at $p = 1$, and \appref{app:certify} for a given pair $(\kappa, p)$.
 
Let $x \in \mathcal{C}^\sigma$ with $\rho > r$ and $\Delta v_\parallel > 0$, and write $\zeta := \sigma\Delta v_\perp$, so that $\zeta \ge 0$ by the definition of $\mathcal{C}^\sigma$, and $\zeta \le \|\Delta\bm{v}\| \le w_{\max}$. The exponent enters through $d^{\,p-1}$, which we bound on the range of distances considered by two constants, $\bar D \ge d^{\,p-1}$ and $M_p \ge d^{\,p-1}(p + r^2/\rho^2)$. At $p = 1$ these are $\bar D = 1$ and $M_1 = 2$, for every $\rho > r$. For $p > 1$, $d^{\,p-1}$ grows without bound with $\rho$, so the range must be bounded, and \appref{app:certify} gives the constants for $p > 1$. The following lemma bounds the drift below by \(-\rho v^2\nu\), where
\[
\nu := \kappa w_{\max} M_p + r\kappa^2 p\bar D^2 v_{\max}^2.
\]
 
\begin{lemma}
\label{lem:drift}
The drift satisfies
\begin{equation}
\begin{aligned}
L_f h_\sigma &\ \ge\ -\kappa d^{\,p-1} v^2 \frac{p\rho^2 + r^2}{\rho}\,\zeta \;-\; r\kappa^2 p\, d^{\,2p-2}\rho v^4 \\
&\ \ge\ -\rho v^2 \nu .
\end{aligned}
\label{eq:driftbound}
\end{equation}
\end{lemma}
 
\begin{proof}
By \eqref{eq:hdot}, $L_f h_\sigma = \tfrac{r\zeta^2}{\rho} - \tfrac{r\Delta v_\parallel}{d}\bigl(\tfrac{r\zeta}{\rho} + \kappa p\, d^{\,p-1}\rho v^2\bigr)$, and the factor in brackets is nonnegative since $\zeta \ge 0$. We write $h_\sigma \ge 0$ as $-r\Delta v_\parallel + d\zeta + r\kappa d^{\,p} v^2 \ge 0$ and divide by $d > 0$:
\begin{equation}
\frac{r\Delta v_\parallel}{d} \ \le\ \zeta + r\kappa d^{\,p-1} v^2 .
\label{eq:cancellation}
\end{equation}
We multiply \eqref{eq:cancellation} by $\bigl(\tfrac{r\zeta}{\rho} + \kappa p\, d^{\,p-1}\rho v^2\bigr) \ge 0$, and expand the right side:
\[
\begin{aligned}
&(\zeta + r\kappa d^{\,p-1} v^2)\Bigl(\tfrac{r\zeta}{\rho} + \kappa p\, d^{\,p-1}\rho v^2\Bigr) \\
&\quad = \tfrac{r\zeta^2}{\rho} + \kappa d^{\,p-1} v^2\,\tfrac{p\rho^2+r^2}{\rho}\,\zeta + r\kappa^2 p\, d^{\,2p-2}\rho v^4 .
\end{aligned}
\]
Substituting this upper bound into the expression for $L_f h_\sigma$, the term $r\zeta^2/\rho$ cancels, which gives the first inequality of \eqref{eq:driftbound}. For the second, $\tfrac{p\rho^2 + r^2}{\rho} = \rho\,(p + r^2/\rho^2)$, and we use $\zeta \le w_{\max}$, $d^{\,p-1}(p + r^2/\rho^2) \le M_p$, $d^{\,p-1} \le \bar D$ and $v \le v_{\max}$.
\end{proof}
 
\emph{The authority.} The following lemma gives a lower bound on the authority \(\Phi_\sigma\), again proportional to \(\rho v^2\), so that it can be compared directly with the drift bound.
 
\begin{lemma}
\label{lem:auth}
Let \(\psi:=\arccos(r/\rho)\), \(\varphi:=\psi+\sigma\tilde\theta\), \(\varepsilon:=\nu/a_{\max}\), and \(\xi:=\varepsilon v_{\max}^2+2r\kappa\bar Dv_{\max}\). The coefficients satisfy $c^a_\sigma = -\rho\cos\varphi$ and $B_\sigma = \sigma v\,(Q\sin\varphi - P\cos\varphi)$, where
\begin{equation}
\begin{aligned}
P &:= \tfrac{d(d\zeta - r\Delta v_\parallel)}{\rho^2} + r\kappa p\, d^{\,p-1} v^2 ,\\
Q &:= \tfrac{v\rho}{\ell_r} + \tfrac{d(r\zeta + d\Delta v_\parallel)}{\rho^2} + \tfrac{r\zeta}{d} + r^2\kappa p\, d^{\,p-2} v^2 .
\end{aligned}
\label{eq:PQ}
\end{equation}
If $\xi < 1$, then
\begin{equation}
\begin{aligned}
\Phi_\sigma &\ \ge\ \rho v^2 \min\bigl(\nu,\ \beta_{\max}\Lambda\bigr), \quad \text{where}\\
\Lambda &:= \tfrac{\sqrt{1 - \xi^2}}{\ell_r} - \tfrac{1}{r}\bigl(w_{\max} + r\kappa p\bar D v_{\max}^2\bigr)\bigl(\varepsilon v_{\max} + 2r\kappa\bar D\bigr) .
\end{aligned}
\label{eq:authmin}
\end{equation}
\end{lemma}
 
\begin{proof}
Since $r^2 + d^2 = \rho^2$, we have $r = \rho\cos\psi$ and $d = \rho\sin\psi$. We substitute $\sin\tilde\theta = \sigma\sin(\varphi - \psi)$ and $\cos\tilde\theta = \cos(\varphi - \psi)$ into $c^a_\sigma$ and $B_\sigma$ of \eqref{eq:hdot}, expand, and collect the coefficients of $\sin\varphi$ and $\cos\varphi$, which gives the two identities. On $\mathcal{C}^\sigma$, $h_\sigma \ge 0$ gives $d\zeta - r\Delta v_\parallel \ge -r\kappa d^{\,p} v^2$, so $P \ge r\kappa d^{\,p-1}v^2(p - d^2/\rho^2) \ge 0$. Since $\Delta v_\parallel > 0$, dropping $-rd\Delta v_\parallel/\rho^2$ gives $P \le \zeta + r\kappa p\, d^{\,p-1}v^2 \le w_{\max} + r\kappa p\bar D v_{\max}^2$, and every term of $Q$ is nonnegative, so $Q \ge v\rho/\ell_r$. We consider two cases.
 
If $|c^a_\sigma + S| \ge \varepsilon\rho v^2$, then $\Phi_\sigma \ge a_{\max}\varepsilon\rho v^2 = \rho v^2\nu$.
 
Otherwise, since $S = 2r\kappa d^{\,p-1} d\, v \le 2r\kappa\bar D\rho v$, we have $\rho|\cos\varphi| \le |c^a_\sigma + S| + S < \varepsilon\rho v^2 + 2r\kappa\bar D\rho v$, so $|\cos\varphi| < \xi$ and $|\sin\varphi| > \sqrt{1 - \xi^2}$. Then
\[
\begin{aligned}
|B_\sigma| &\ \ge\ v\bigl(Q|\sin\varphi| - P|\cos\varphi|\bigr) \\
&\ \ge\ v^2\Bigl(\tfrac{\rho\sqrt{1 - \xi^2}}{\ell_r} - P(\varepsilon v + 2r\kappa\bar D)\Bigr) ,
\end{aligned}
\]
and using the upper bound on $P$, $v \le v_{\max}$ and $\rho \ge r$ gives $|B_\sigma| \ge \rho v^2\Lambda$, so $\Phi_\sigma \ge \beta_{\max}\rho v^2\Lambda$.
\end{proof}
 
\emph{The resulting condition.} Adding the bounds of \lemref{lem:drift} and \lemref{lem:auth} gives $\sup_{\bm{u}}\dot h_\sigma \ge \rho v^2[\min(\nu, \beta_{\max}\Lambda) - \nu]$. So \eqref{eq:key} holds at every state of $\mathcal{C}^\sigma$ with $\rho > r$ in the range of distances considered if
\begin{equation}
\textbf{(C$\kappa$)} \qquad \xi < 1 \quad \text{and} \quad \beta_{\max}\Lambda \ \ge\ \nu .
\label{eq:G}
\end{equation}
This condition involves only the parameters, not the state. Every term of the two bounds is proportional to $v^2$, so no lower bound on the speed is needed, and $v = 0$ is included.
 
\subsection{A Closed-Form Gain}
\label{app:closedform}
 
At $p = 1$, \eqref{eq:G} depends on $\kappa$ through $\nu$, $\varepsilon$ and $\xi$. Bounding each of them linearly in $\kappa$ gives an explicit gain.
 
\begin{lemma}
\label{lem:threshold}
Let $p = 1$, so that $\bar D = 1$, $M_1 = 2$ and $\nu = 2\kappa w_{\max} + r\kappa^2 v_{\max}^2$. Let
\begin{equation}
\begin{aligned}
g_1 &:= \tfrac{3 w_{\max} v_{\max}^2}{a_{\max}} + 2 r v_{\max} ,\\
g_2 &:= 2 w_{\max}\Bigl(\tfrac{3 w_{\max} v_{\max}}{a_{\max}} + 2r\Bigr) ,\\
\kappa^\star &:= \min\Bigl(\tfrac{w_{\max}}{r v_{\max}^2},\ \Bigl(g_1 + \tfrac{\ell_r g_2}{r} + \tfrac{3\ell_r w_{\max}}{\beta_{\max}}\Bigr)^{-1}\Bigr) .
\end{aligned}
\label{eq:kappastar}
\end{equation}
Then $\kappa^\star > 0$, and \eqref{eq:G} holds for every $0 < \kappa \le \kappa^\star$.
\end{lemma}
 
\begin{proof}
Every term in \eqref{eq:kappastar} is positive, so $\kappa^\star > 0$. Let $0 < \kappa \le \kappa^\star$. Then $\kappa$ is at most each of the two terms of the minimum in \eqref{eq:kappastar}, so
\[
\begin{aligned}
r\kappa v_{\max}^2 &\le w_{\max} ,\\
\kappa\Bigl(g_1 + \tfrac{\ell_r g_2}{r} + \tfrac{3\ell_r w_{\max}}{\beta_{\max}}\Bigr) &\le 1 .
\end{aligned}
\]
The first inequality gives $\nu \le 3 w_{\max}\kappa$ and $w_{\max} + r\kappa v_{\max}^2 \le 2 w_{\max}$. So $\varepsilon \le 3 w_{\max}\kappa/a_{\max}$, which gives $\xi \le \kappa g_1$ and $(w_{\max} + r\kappa v_{\max}^2)(\varepsilon v_{\max} + 2r\kappa) \le \kappa g_2$. In the second inequality, the last two terms in the bracket are positive, so $\kappa g_1 < 1$, and hence $\xi < 1$. It remains to show $\beta_{\max}\Lambda \ge \nu$. At $p = 1$,
\[
\Lambda = \tfrac{\sqrt{1 - \xi^2}}{\ell_r} - \tfrac{1}{r}\bigl(w_{\max} + r\kappa v_{\max}^2\bigr)\bigl(\varepsilon v_{\max} + 2r\kappa\bigr) .
\]
Since $\sqrt{1 - x^2} \ge 1 - x$ for $x \in [0, 1]$, $\xi \le \kappa g_1$ gives $\sqrt{1 - \xi^2} \ge 1 - \kappa g_1$. The second term is at most $\kappa g_2/r$ by the bound above. Hence $\Lambda \ge (1 - \kappa g_1)/\ell_r - \kappa g_2/r$, and with $\nu \le 3w_{\max}\kappa$,
\[
\begin{aligned}
\beta_{\max}\Lambda - \nu &\ \ge\ \beta_{\max}\Bigl(\tfrac{1 - \kappa g_1}{\ell_r} - \tfrac{\kappa g_2}{r}\Bigr) - 3w_{\max}\kappa \\
&\ =\ \tfrac{\beta_{\max}}{\ell_r}\Bigl[1 - \kappa\Bigl(g_1 + \tfrac{\ell_r g_2}{r} + \tfrac{3\ell_r w_{\max}}{\beta_{\max}}\Bigr)\Bigr] ,
\end{aligned}
\]
where the equality takes out the factor $\beta_{\max}/\ell_r$. The bracket is nonnegative by the second inequality above.
\end{proof}

\subsection{Completing the Proof}
\label{app:propproof}
 
Take $p = 1$ and $0 < \kappa \le \kappa^\star$, which is possible because $\kappa^\star > 0$. By \lemref{lem:threshold}, \eqref{eq:G} holds, so \eqref{eq:key} holds at every state of $\mathcal{C}^\sigma$ with $\rho > r$. Since $h_\sigma \ge 0$ on $\mathcal{C}^\sigma$ and $\alpha$ is extended class-$\mathcal{K}$, $\alpha(h_\sigma) \ge 0$, so $\sup_{\bm{u}}\dot h_\sigma \ge 0 \ge -\alpha(h_\sigma)$. At a state of $\mathcal{C}$ with $\rho > r$, the edge selected by \eqref{eq:ours_qp} is the one with the larger $h_\sigma$, so the state lies in $\mathcal{C}^\sigma$ for that edge, and \eqref{eq:ours_qp} is feasible there. \hfill$\blacksquare$
 
\section{Proof of \thmref{thm:safety}}
\label{app:safety}
 
We prove a stronger statement: the theorem holds for every absolutely continuous trajectory whose input satisfies the constraints of \eqref{eq:ours_qp} at almost every time. We show that the state stays in $\mathcal{C}$, that is, $\rho \ge r$ and $h \ge 0$ at every time. Since $h$ contains $d = \sqrt{\rho^2 - r^2}$, it is defined only where $\rho \ge r$, so $h \ge 0$ at every time already includes $\rho \ge r$. The barrier constraint, however, controls $h$ only where $h$ is defined, and cannot alone prevent entry into the obstacle: at contact, $h \ge 0$ gives only $\Delta v_\parallel \le 0$, which by \eqref{eq:rhodot} does not prevent $\dot\rho < 0$. We therefore first show $\rho \ge r$ from the contact constraints (Step 1), and then show that $h$ stays nonnegative by combining the barrier constraint, which holds away from contact, with the contact constraints, which hold at contact (Step 2). Away from contact, $h = \max(h_{+1}, h_{-1})$ is not differentiable where the two edges are equal, so we first need a lemma on the maximum of two functions. At contact, the edges contain $d = \sqrt{\rho^2 - r^2}$, which is not differentiable at $\rho = r$, so there we work instead with $-r\Delta v_\parallel$, which equals $h$ at contact.
 
\begin{lemma}
\label{lem:maxac}
Let $f_1, f_2 : [0,T] \to \mathbb{R}$ be absolutely continuous and $F = \max(f_1, f_2)$. Then $F$ is absolutely continuous, and at almost every $t$, $\dot F(t) = \dot f_i(t)$ for every $i$ with $f_i(t) = F(t)$.
\end{lemma}
 
\begin{proof}
We write $F = \tfrac12(f_1 + f_2 + |f_1 - f_2|)$. Sums and absolute values of absolutely continuous functions are absolutely continuous, so $F$ is. Then $f_1, f_2$ and $F$ are differentiable at almost every $t$. Fix such a $t$. If $f_1(t) > f_2(t)$, then $F = f_1$ near $t$ and the claim follows, and the same holds if $f_2(t) > f_1(t)$. If $f_1(t) = f_2(t)$, then $F(s) - F(t) = \max_i\{f_i(s) - f_i(t)\}$ for every $s$. We divide by $s - t$ and let $s \downarrow t$ to get $\dot F(t) = \max(\dot f_1(t), \dot f_2(t))$. If we let $s \uparrow t$ instead, $s - t$ is negative, so dividing turns the maximum into a minimum, and we get $\dot F(t) = \min(\dot f_1(t), \dot f_2(t))$. Since $F$ is differentiable at $t$, the two limits are equal, so $\dot f_1(t) = \dot f_2(t) = \dot F(t)$.
\end{proof}
 
\emph{Proof of the theorem.} Let $x(\cdot)$, defined on $[0, T]$, be such a trajectory of the dynamics of \secref{subsec:bicycle}, and write $h(t)$ for $h(x(t))$. On every compact interval on which $\rho > 0$, $\rho$ and $\Delta v_\parallel$ are absolutely continuous, with derivatives \eqref{eq:rhodot} and \eqref{eq:ddvpar} at almost every time. Each $h_\sigma$ is continuous on $\{\rho \ge r\}$, and at contact $d = 0$, so every edge equals $-r\Delta v_\parallel$ and
\begin{equation}
h = -r\Delta v_\parallel \qquad \text{wherever } \rho = r .
\label{eq:hcontact}
\end{equation}
 
\emph{Step 1: $\rho \ge r$.} The distance $\rho = \|\Delta\bm{p}\|$ is continuous, since $x$ is. Suppose $\rho(t^\ast) < r$ for some $t^\ast$. Since $x(0) \in \mathcal{C}$, $\rho(0) \ge r$, so $t_r := \sup\{t \le t^\ast : \rho(t) \ge r\}$ exists, and by continuity $\rho(t_r) = r$ and $\rho < r$ on $(t_r, t^\ast]$. Again by continuity, $\rho > 0$ on $[t_r, t_r + \delta]$ for some $\delta > 0$. We replace $t^\ast$ by $\min(t^\ast, t_r + \delta)$, which keeps $\rho(t^\ast) < r$, so that $0 < \rho \le r$ on $[t_r, t^\ast]$. On this interval the input satisfies the contact constraints of \eqref{eq:ours_qp}, so $\dot\rho \ge 0$ at almost every time. Since $\rho$ is absolutely continuous there, $\rho(t^\ast) \ge \rho(t_r) = r$, which contradicts $\rho(t^\ast) < r$.
 
\emph{Step 2: $h \ge 0$.} The robot may touch the obstacle and leave it many times, even infinitely often, so we argue by contradiction on an interval where $h < 0$, and show that $h$ cannot decrease there, either away from contact or across contact. By Step 1, the times split into the closed contact set $K := \{t : \rho(t) = r\}$ and its complement, a union of disjoint open intervals on which $\rho > r$.
 
Away from contact we use the barrier constraint. On every compact subinterval of an interval of the complement, $x$ stays in a compact subset of $\{\rho > r\}$, on which each $h_\sigma$ is continuously differentiable. So each $h_\sigma \circ x$ is absolutely continuous there, and by \lemref{lem:maxac} so is $h \circ x$, with derivative $\dot h_{\sigma(t)}$, where $\sigma(t)$ is the edge whose constraint the input satisfies. That constraint gives $\tfrac{d}{dt}h \ge -\alpha(h)$ at almost every time of the interval.
 
At contact we use the contact constraints. Define $G(t) := -r\Delta v_\parallel(t)$ for all $t$. The function $G$ is absolutely continuous, it equals $h$ on $K$ by \eqref{eq:hcontact}, and at almost every $t \in K$ the input satisfies the contact constraints, so $\tfrac{d}{dt}\Delta v_\parallel \le 0$ and $\dot G \ge 0$.
 
Now suppose $h(t_1) < 0$ for some $t_1$. Since $h(0) \ge 0$, $t_0 := \sup\{t \le t_1 : h(t) \ge 0\}$ exists, and by continuity $h(t_0) = 0$ and $h < 0$ on $(t_0, t_1]$. We show that $h$ is nondecreasing on $[t_0, t_1]$, which contradicts $h(t_1) < 0$. On the part in $(t_0, t_1]$ of each interval of the complement of $K$, $h < 0$, so $\alpha(h) < 0$ and $\tfrac{d}{dt}h > 0$ at almost every time. Hence $h$ is nondecreasing on that part, and by continuity on its closure. For $s < t$ in $K \cap [t_0, t_1]$, the set $[s, t] \setminus K$ is a union of disjoint open intervals $(a, b)$ with $a, b \in K$, and the absolute continuity of $G$, with the countable additivity of the integral, gives
\[
G(t) - G(s) = \int_{K \cap [s,t]} \dot G \, dt' + \sum_{(a,b)} \bigl(G(b) - G(a)\bigr) .
\]
The integral is nonnegative, and each $G(b) - G(a) = h(b) - h(a) \ge 0$, so $h(t) \ge h(s)$. For general $s < t$ in $[t_0, t_1]$, either $[s, t]$ contains no point of $K$, and lies in the closure of one interval of the complement, or $s' := \min(K \cap [s,t])$ and $t' := \max(K \cap [s,t])$ exist, and $h(s) \le h(s') \le h(t') \le h(t)$. So $h$ is nondecreasing on $[t_0, t_1]$.
 
By Steps 1 and 2, $\rho(t) \ge r$ and $h(t) \ge 0$ for all $t \in [0, T]$, which means $x(t) \in \mathcal{C}$. This holds for every $T$, which gives the theorem. \hfill$\blacksquare$
 
\section{Certification of the Gain}
\label{app:certify}
 
The gain $\kappa^\star$ of \appref{app:closedform} is conservative, since \eqref{eq:G} bounds the drift and the authority separately, allowing worst cases that need not occur at the same state. For a given pair $(\kappa, p)$, and for $p > 1$ a range $\rho \le \rho_{\max}$, the following three steps verify \eqref{eq:key} at every state of $\mathcal{C}^\sigma$ with $\rho > r$ in that range.
 
\emph{Step 1: constants.} At $p = 1$, $\bar D = 1$ and $M_1 = 2$. For $p > 1$ we take $\bar D = \rho_{\max}^{\,p-1}$, and write $d^{\,p-1}(p + r^2/\rho^2) = p\,d^{\,p-1} + d^{\,p-1}r^2/(d^2 + r^2)$. The second term is largest at $d = r\sqrt{(p-1)/(3-p)}$ for $1 < p < 3$, which gives $M_p = p\bar D + \tfrac{3-p}{2}\bigl(r\sqrt{(p-1)/(3-p)}\bigr)^{p-1}$.
 
\emph{Step 2: low speeds.} The bounds of \appref{app:bounds} use the speed only through $v \le v_{\max}$, so on the states with $v \le v_c$, for any $v_c \le v_{\max}$, they hold with $v_c$ in place of $v_{\max}$, also in $w_{\max}$. With this replacement, $\nu$ and $\xi$ increase with $v_c$ and $\Lambda$ decreases, so if \eqref{eq:G} holds at some $v_c$, it holds at every smaller one. We find the largest such $v_c$, up to a tolerance, by bisection on $[0, v_{\max}]$, keeping the end of the interval at which \eqref{eq:G} holds. Then \eqref{eq:key} holds at all speeds $v \le v_c$. If \eqref{eq:G} holds at no $v_c > 0$, the pair is not verified.
 
\emph{Step 3: remaining speeds.} For $v_c \le v \le v_{\max}$ we divide the states into boxes, with one interval for each of $r/\rho$, the heading, $v$ and $\sigma\Delta v_\perp$. For $\Delta v_\parallel$ we use the range allowed by $h_\sigma \ge 0$ and $\|\bm{v}_o\| \le v_{o,\max}$. On each box, we compute a lower bound on $\sup_{\bm{u}}\dot h_\sigma$, with $B_\sigma$ as in \eqref{eq:hdot}, by replacing every quantity with the worst value its interval allows, even if these values never occur together at one state. We compute this bound in interval arithmetic and round each result outward, so the bound stays valid in floating point.
 
If the bound on a box is nonnegative, every state in the box satisfies \eqref{eq:key}, and we remove the box from the list. We also remove any box that contains no state of $\mathcal{C}^\sigma$. If the bound is negative, we split the box in two and put both halves back on the list. When the list is empty, every state of the domain lies in a removed box, so \eqref{eq:key} holds on the whole domain, and the pair is verified. If instead a box reaches a minimum width with a negative bound, or the list exceeds a size limit, the pair is not verified. A smaller $v_c$ would only enlarge this domain.

\end{document}